\documentclass{llncs}

\usepackage{hyperref}

\usepackage{cite}
\usepackage{graphicx}
\usepackage{framed}
\usepackage{amsfonts}
\usepackage{amsmath}
\usepackage{color}
\usepackage{algorithm}
\usepackage[noend]{algpseudocode}
\definecolor{mygreen}{rgb}{0,0.2,0}
\definecolor{mygray}{rgb}{0.95,0.95,0.95}
\definecolor{mymauve}{rgb}{0.58,0,0.82}

\definecolor{lbcolor}{rgb}{0.95,0.95,0.95}

\usepackage{listings,newtxtt}
\begin{document}

\title{Towards Automated Discovery of \\
Geometrical Theorems in GeoGebra}
\author{Zolt\'an Kov\'acs\orcidID{0000-0003-2512-5793}\inst{1} \and Jonathan H.~Yu\inst{2}}
\institute{
The Private University College of Education of the Diocese of Linz\\
Salesianumweg 3, A-4020 Linz, Austria\\
\email{zoltan@geogebra.org}\and
Gilman School, Baltimore, Maryland, USA\\
\email{jonathanhy314@gmail.com}
}

\maketitle              

\begin{abstract}
We describe a prototype of a new experimental GeoGebra command and tool \texttt{Discover}
that analyzes geometric figures for salient patterns, properties, and theorems.
This tool is a basic implementation of automated discovery in elementary planar geometry.
The paper focuses on the mathematical background of the implementation,
as well as  methods to avoid combinatorial explosion when storing
the interesting properties of a geometric figure.

\keywords{GeoGebra, automated reasoning, combinatorial explosion, equivalence relation}

\end{abstract}
\section{Introduction}

In this technical paper we introduce a new GeoGebra command and tool \texttt{Discover}
that is available in a development GitHub repository \cite{geogebra-discovery}.
This research is closely related to a former
project \cite{ag} (see \cite{adg-ag,aisc-ag,LNAI11110-ag} for further details).

Given a Euclidean geometry construction drawn in GeoGebra, suppose a user
wants to know if a given object $O$ has some ``interesting features,''
such as relevant theorems or properties.
This object can be a point, a line, a circle, or something else, although
in the current implementation $O$ will always be a point. 
Without any further user input,
the \texttt{Discover} command will then analyze $O$
for its interesting and relevant features, and present them to the user
as both a list of formulas and graphics outputs.

\begin{figure}\centering
\includegraphics[scale=0.75]{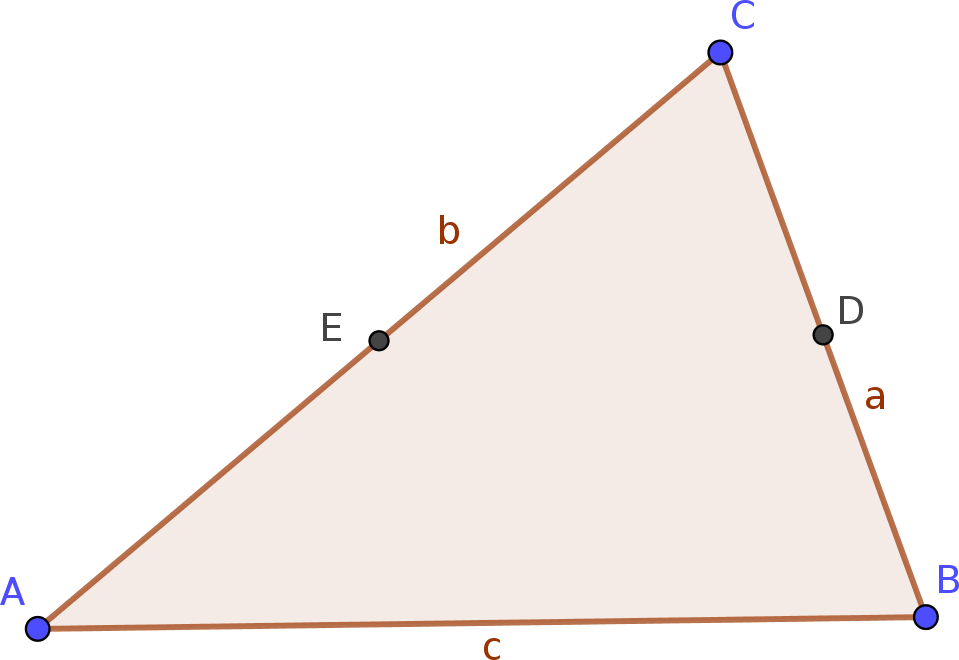}
\caption{Initial setup for a discovery}
\label{midline1}
\end{figure}

For example, let $ABC$ an arbitrary triangle, and let $D$ and $E$ be
the midpoints of $BC$ and $AC$, respectively (Fig.~\ref{midline1}). Has $D$ some interesting features?
Yes: $DE$ is parallel to $AB$, independent of the position of $A$, $B$ and $C$.
Indeed, the command \texttt{Discover($D$)} confirms this observation with
the output shown in Fig.~\ref{rel-midline}; GeoGebra adds lines $DE$ and $AB$
in the same color (Fig.~\ref{midline2}).
(Note, however, that the current implementation of GeoGebra does not report that $2\cdot|DE|=|AB|$.)
Also, the software reports the somewhat trivial finding that the segments $BD$ and $CD$
are congruent, with $BD$ and $CD$ highlighted in the same color.
This output can be
obtained by selecting the Discover tool in GeoGebra's toolbox:
\begin{center}
\includegraphics[height=0.8cm]{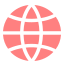}
\end{center}
and then clicking on the point $D$. This functionality is implemented in both GeoGebra Classic 5 and 6,
available as an experimental software package called \textit{GeoGebra Discovery}, at
\url{http://github.com/kovzol/geogebra-discovery}.

\begin{figure}\centering
{\setlength{\fboxsep}{0pt}%
\setlength{\fboxrule}{0.5pt}%
\fbox{\includegraphics[scale=0.4]{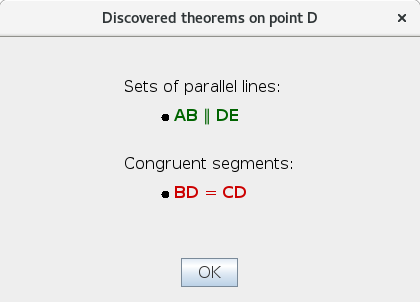}}%
}
\caption{Output window of the \texttt{Discover} command that reports the \textit{Midline theorem}}
\label{rel-midline}
\end{figure}

\begin{figure}\centering
\includegraphics[scale=0.75]{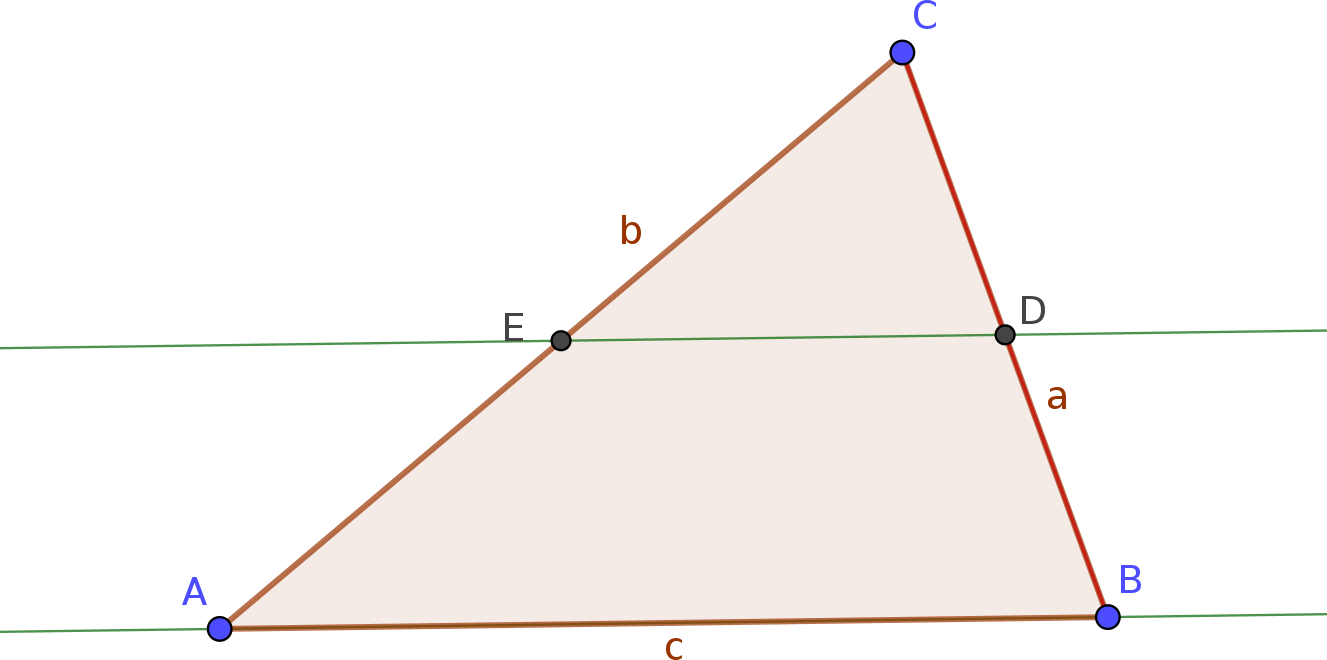}
\caption{Further output of the \texttt{Discover} command}
\label{midline2}
\end{figure}

What strategy is used in the background? First, all points are analyzed to determine whether they are
the same as another point. Then, all possible point triplets are examined for collinearity.
Next, all possible subsets containing four points on the figure are checked
for concyclicity. With knowledge of the collinear points, separate lines can be uniquely defined, in order
find whether they are parallel. Finally, considering the pairs of all possible point
pairs, congruent segments can be identified. This strategy is a result of a combination of numerical
and symbolic processes.

Our second example shows a more complicated setup. A regular hexagon  $ABCDEF$ is given in Fig.~\ref{hexagon1}.
Point $G$ is defined as the intersection of $AD$ and $BE$, and, in addition, $H=BE\cap CF$, $I=AD\cap CF$.
The points $G$, $H$ and $I$ may have trivial differences in their numerical representations, but
in the geometrical sense they should be equal. In the figure rounding was set to 13 digits to emphasize that
GeoGebra computes objects numerically by default. Note that while the $y$-coordinates of $H$ and $I$ numerically differ, the final calculations to prove that they are identical will be symbolic and exact.

\begin{figure}\centering
{\setlength{\fboxsep}{0pt}%
\setlength{\fboxrule}{0.5pt}%
\fbox{\includegraphics[scale=0.4]{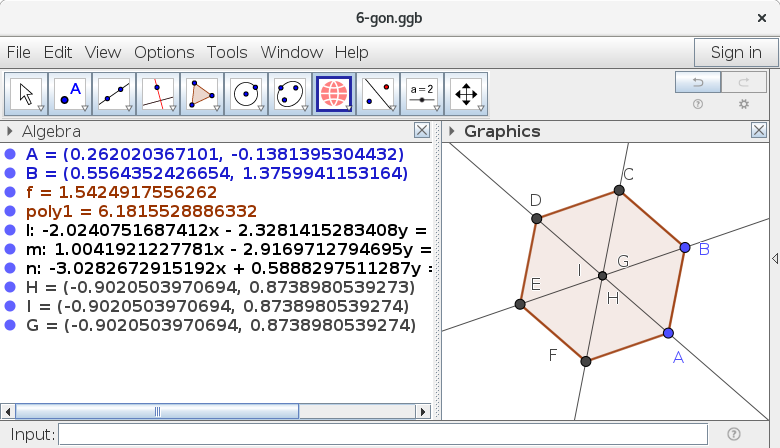}}%
}
\caption{Initial setup for another discovery}
\label{hexagon1}
\end{figure}

Now we are about to learn if point $F$ has some interesting features, so the command
\texttt{Discover($F$)} will be issued. GeoGebra reports a set of properties in a message box
(Fig.~\ref{rel-hexagon}) and adds some additional outputs to the initial setup (Fig.~\ref{hexagon2}).

\begin{figure}\centering
{\setlength{\fboxsep}{0pt}%
\setlength{\fboxrule}{0.5pt}%
\fbox{\includegraphics[scale=0.4]{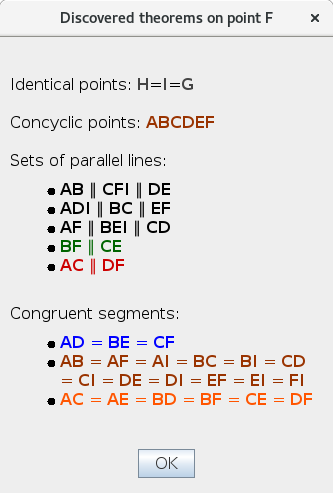}}%
}
\caption{Output window that reports several theorems related to point $F$}
\label{rel-hexagon}
\end{figure}

\begin{figure}\centering
\includegraphics[scale=0.75]{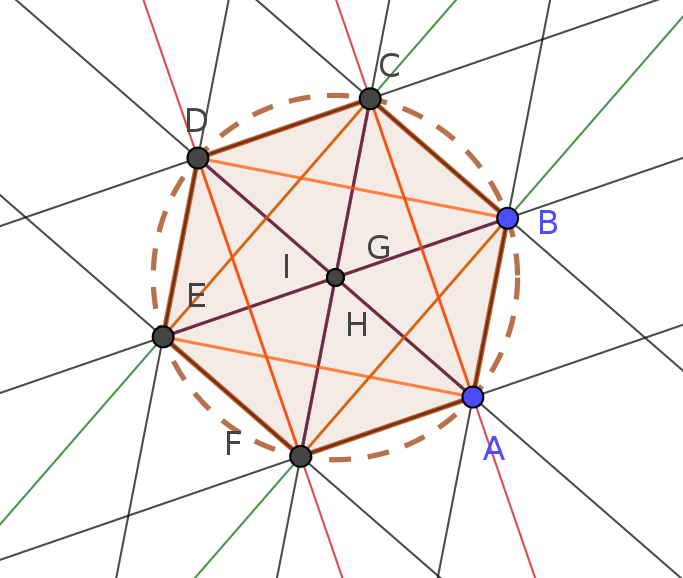}
\caption{Further output of discovery}
\label{hexagon2}
\end{figure}

Here, we see that concyclic points are reported as a single item and not as separate data.
Also, parallel lines are classified into five different sets.
Finally, there are three sets of congruent segments.
This approach in computation and reporting helps avoid combinatorial explosion.

\section{Mathematical background}\label{sec2}

The above mentioned strategies have some similiarities to the ones introduced in \cite{song}, but here we
focus on minimizing the number of objects that have to be compared in the process
that practically compares all objects with all other objects.

Our current implementation deals with \textit{points}, \textit{lines},
\textit{circles} and \textit{parallel lines} (or \textit{directions}) and \textit{congruent segments}.

A \textit{geometric point} $P$ is a GeoGebra object, described by the \texttt{GeoPoint} class
(see GeoGebra's source code at \url{github.com/geogebra/geogebra} for more details). While we will not provide
a detailed definition of a geometric point, generally speaking it is an object with a very complex structure
containing two real coordinates, several style settings (including size and color, for example)
and other technical details that are used in the application. Some geometric points are
dependent of other geometric points or other geometric objects---this hierarchy is stored in the set of \texttt{GeoPoint}s, too.

Independent of the detailed definition of a geometric point,
we can still define the notion of \textit{point} in our context.

\begin{definition} A set of geometric points ${\cal P}=\{P_1,P_2,\ldots,P_n\}$ is called a point if for all
different $P,Q\in {\cal P}$ the points $P$ and $Q$ are identical in general.
\end{definition}

Henceforth, unless otherwise mentioned, we will consider points according to the definition above,
not as geometric points.

Here, we do not precisely define when two points are identical \textit{in general}.
Instead, we will illustrate the concept of point identicality with the following example.
Consider geometric points $P_1$, $P_2$,
$P_3$ and $P_4$ that form a parallelogram. Now define $P_5$ and $P_6$ as the midpoint of $P_1$ and $P_3$, and $P_2$ and $P_4$,
respectively. This setting implies that $P_5$ and $P_6$ are identical, because the diagonals of a parallelogram
always bisect each other. In a dynamic geometry setting like GeoGebra, this simply means that by changing some points
of the set $\{P_1,P_2,P_3,P_4\}$, the points $P_5$ and $P_6$ will still share the same position in the plane. (See
Fig.~\ref{parallelogram1}. Here the construction is controlled by the points $P_1$, $P_2$ and $P_3$ only:
they can be freely chosen, and based on them, the point $P_4$ is already dependent and uniquely defined as
the intersection of the two parallel lines to $P_1P_2$ and $P_2P_3$, respectively, through $P_3$ and $P_1$.)
\begin{figure}\centering
\includegraphics[scale=0.75]{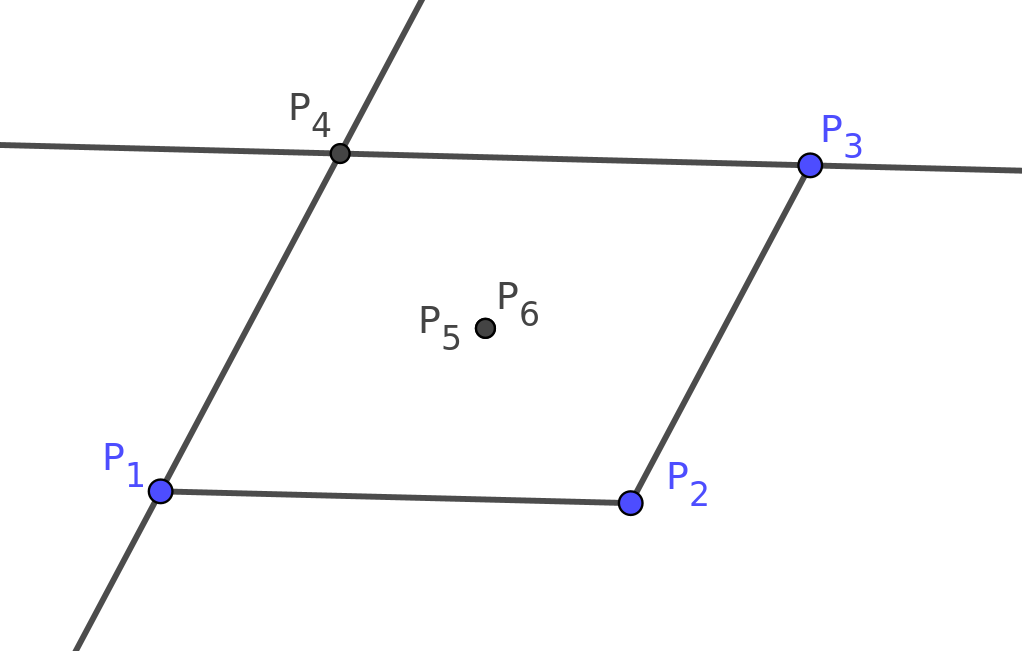}
\caption{Points $P_5$ and $P_6$ are defined as midpoints of opposite vertices of parallelogram $P_1P_2P_3P_4$}
\label{parallelogram1}
\end{figure}

In fact, general truth includes statements that are not always true, but just ``in most cases''---here we can think
of some degeneracies that can occur in some constructions when some objects are degenerate.
For example, altitudes of a triangle generally meet at a point---but not always, since a degenerate triangle
``usually'' has three parallel ``altitudes''; unless two (or even three!) vertices of the triangle coincide.
(See \cite{Chou_1987} for more details on the concept of general truth and degeneracies.)

\begin{definition} A set of points $\ell=\{P_1,P_2,\ldots,P_n\}$ is called a line if for all
different $P,Q,R\in \ell$ the points $P$, $Q$ and $R$ are collinear in general.
\end{definition}


For example, the set $\ell=\{C,F,G\}$ in Fig.~\ref{hexagon2} forms a line.

\begin{definition} A set of points ${\cal C}=\{P_1,P_2,P_3,\ldots,P_n\}$ is called a circle if for all
different $P,Q,R,S\in {\cal C}$ the points $P$, $Q$, $R$ and $S$ are concyclic in general.
\end{definition}

\begin{definition} A set of lines $\vec{D}=\{\ell_1,\ldots,\ell_n\}$ is called parallel lines (or a direction) if for all
different $\ell,m\in \vec{D}$ the lines $\ell$ and $m$ are parallel in general.
\end{definition}

\begin{definition} A set $\overline{s}=\{P,Q\}$ of two points is called a segment.
\end{definition}

\begin{definition} A set of segments $s=\{\overline{s_1},\ldots,\overline{s_n}\}$ is called equal length segments (or congruent segments) if for all
different $\overline{s_1},\overline{s_2}\in s$ the segments $\overline{s_1}$ and $\overline{s_1}$ are equally long in general.
\end{definition}

In fact, GeoGebra Discovery uses
a more general concept of being identical: it allows two points (or two objects) to have a kind of
relationship also if it is true just \textit{on parts} (see \cite{rmc-top} for more details).

The main idea of storing the objects is that points, lines, circles, directions and equally long segments designate equivalence classes, that is:

\begin{theorem}Let $\ell$ and $m$ be lines. Then, for all different points $P,Q,R\in m$,
if $\{P, Q\}\subset \ell$, then $R\in \ell$; that is, $\ell = m$.
\end{theorem}

\begin{proof}
In Euclidean geometry two points always designate a unique line.
\end{proof}

\begin{theorem}Let ${\cal C}$ and ${\cal D}$ be circles. Then, for all different points 
$P,Q,R,S\in{\cal D}$, if $\{P,Q,R\}\subset {\cal C}$, then $S\in{\cal C}$; that is, ${\cal C}={\cal D}$.
\end{theorem}
\begin{proof}
In Euclidean geometry three non-collinear points always designate a unique circle.
\end{proof}

\begin{theorem}Let $\vec{D}$ and $\vec{E}$ be directions. Let $\ell\in \vec{D}$ and $m\in\vec{E}$. If $\ell\parallel m$
in general, then $\vec{D}=\vec{E}$.
\end{theorem}
\begin{proof}
This follows immediately from the transitive property of parallelism.
\end{proof}

\begin{theorem}Let $s$ and $t$ be segments. Let $\overline{u}\in s$ and $\overline{v}\in t$. If $|\overline{u}|=|\overline{v}|$
in general, then $s_1=s_2$.
\end{theorem}
\begin{proof}
This is an immediate consequence of the transitive property of equality of lengths.
\end{proof}

By using these theorems we can maintain a minimal set of objects during discovery.

Fig.~\ref{hexagon-pool} shows the objects identified during the
command \texttt{Discover($F$)} for the input in Fig.~\ref{hexagon1}. The set of lines are not listed in the figure
separately, but as a single entry at the bottom list of equally long segments.
Also, some of the outputs are not particularly interesting features, such as lines with only point, circles with only
three points, directions with only one line, or isolated equal length segments.
\begin{figure}\centering
\includegraphics[scale=0.4]{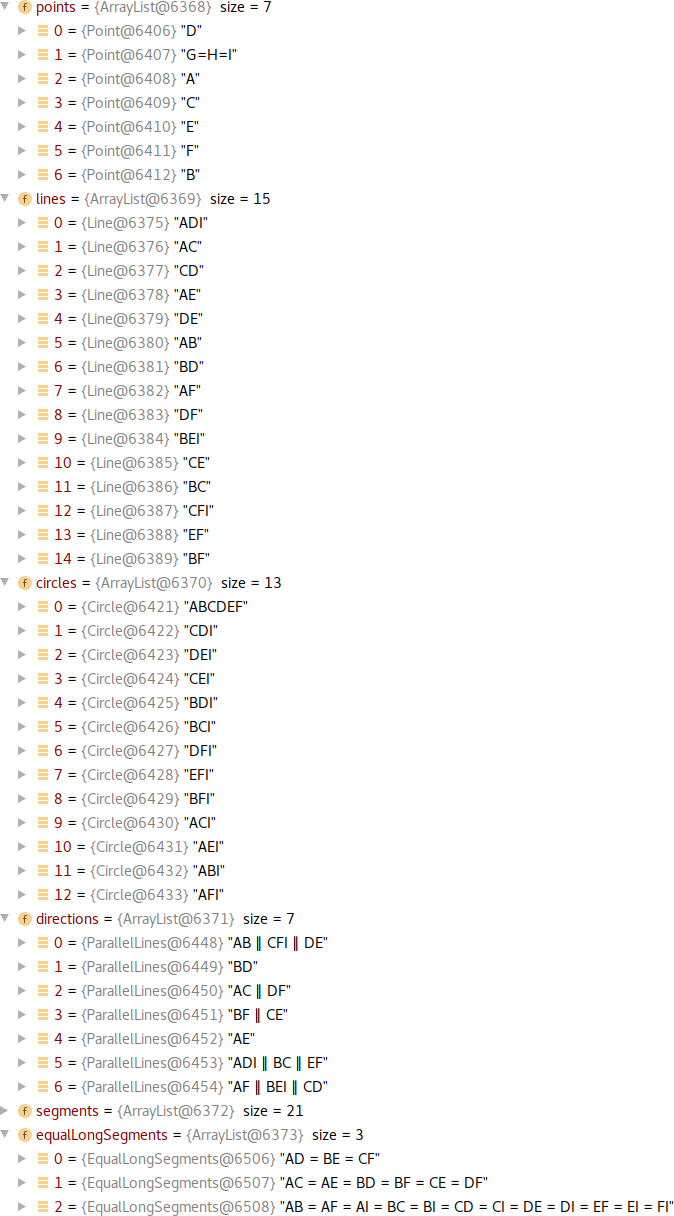}
\caption{The list of objects as shown in IntelliJ IDEA, a popular integrated development environment for Java}
\label{hexagon-pool}
\end{figure}

%
%
%

\section{Examples of Discover with selected theorems}

GeoGebra is a well-known and widely used software tool in education, with meaningful potential for using
geometric discovery and exploration to teach elementary geometry. Even so, the range of mathematical knowledge
is broad, including secondary school topics, international math competitions, and higher level mathematics.
Below we examine selected theorems confirmed in the current implementation of Discover.

\subsection{The diagonals of a parallelogram bisect each other}
We already mentioned this simple theorem. The problem is shown in Fig.~\ref{parallelogram1}.
With discovery on point $P_5$, the applicable theorems are reported in Fig.~\ref{parallelogram2}.
\begin{figure}\centering
\includegraphics[scale=0.75]{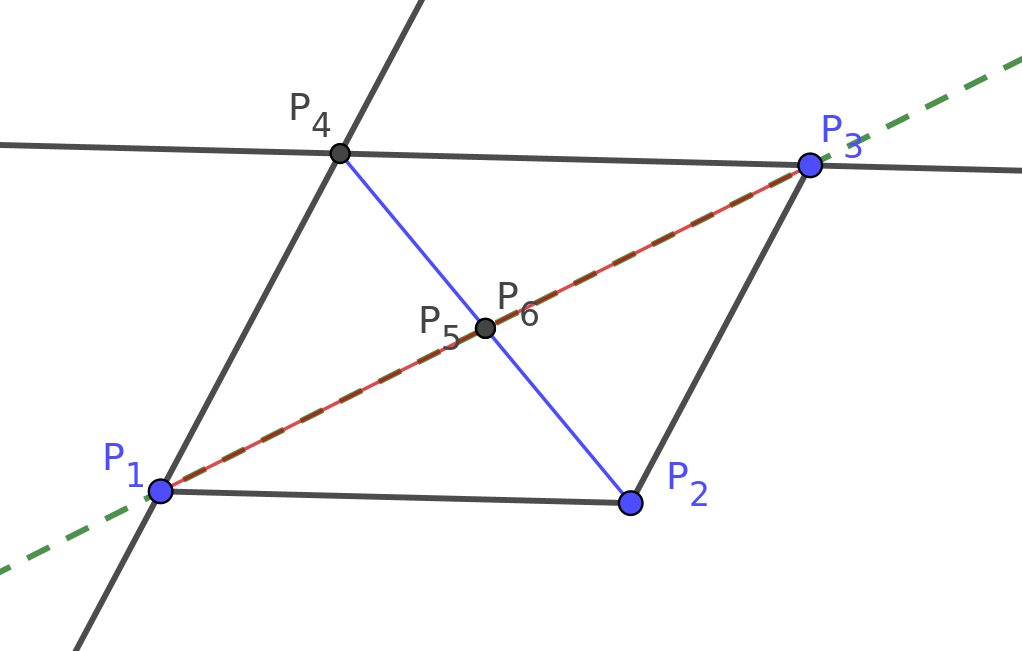}
{\setlength{\fboxsep}{0pt}%
\setlength{\fboxrule}{0.5pt}%
\fbox{\includegraphics[scale=0.4]{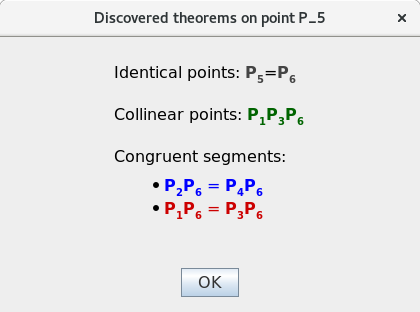}}%
}
\caption{Output of the command \texttt{Discover($P_5$)}}
\label{parallelogram2}
\end{figure}

\subsection{Euler line}

The Euler line is a
line determined from any triangle that is not regular. It
passes through the orthocenter, the circumcenter and the centroid.
The problem is shown in Fig.~\ref{EulerLine1}.
\begin{figure}\centering
\includegraphics[scale=0.75]{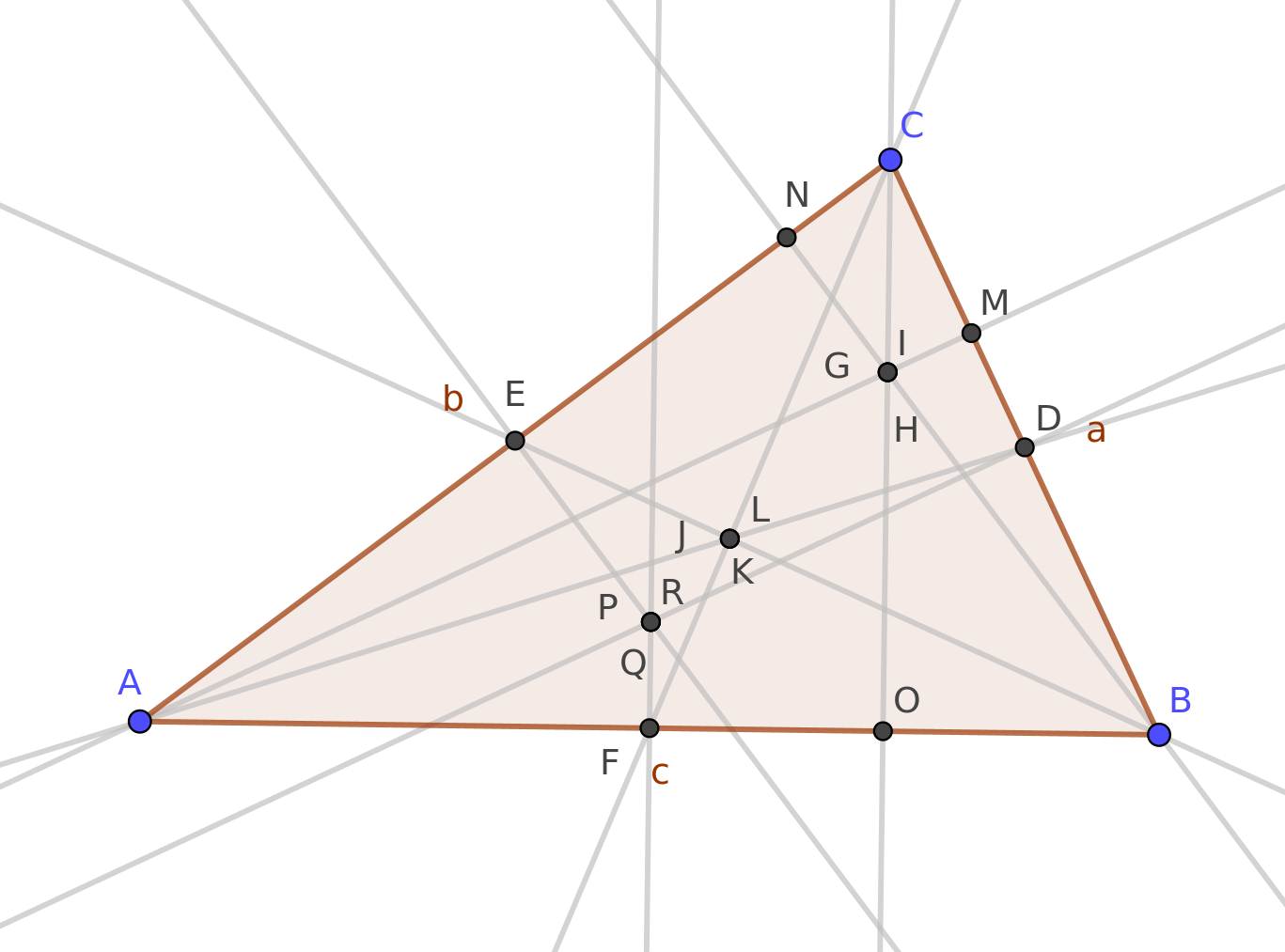}
\caption{Euler line}
\label{EulerLine1}
\end{figure}
With discovery on point $P$, the relevant theorems are listed in Fig.~\ref{EulerLine2}.
\begin{figure}\centering
\includegraphics[scale=0.75]{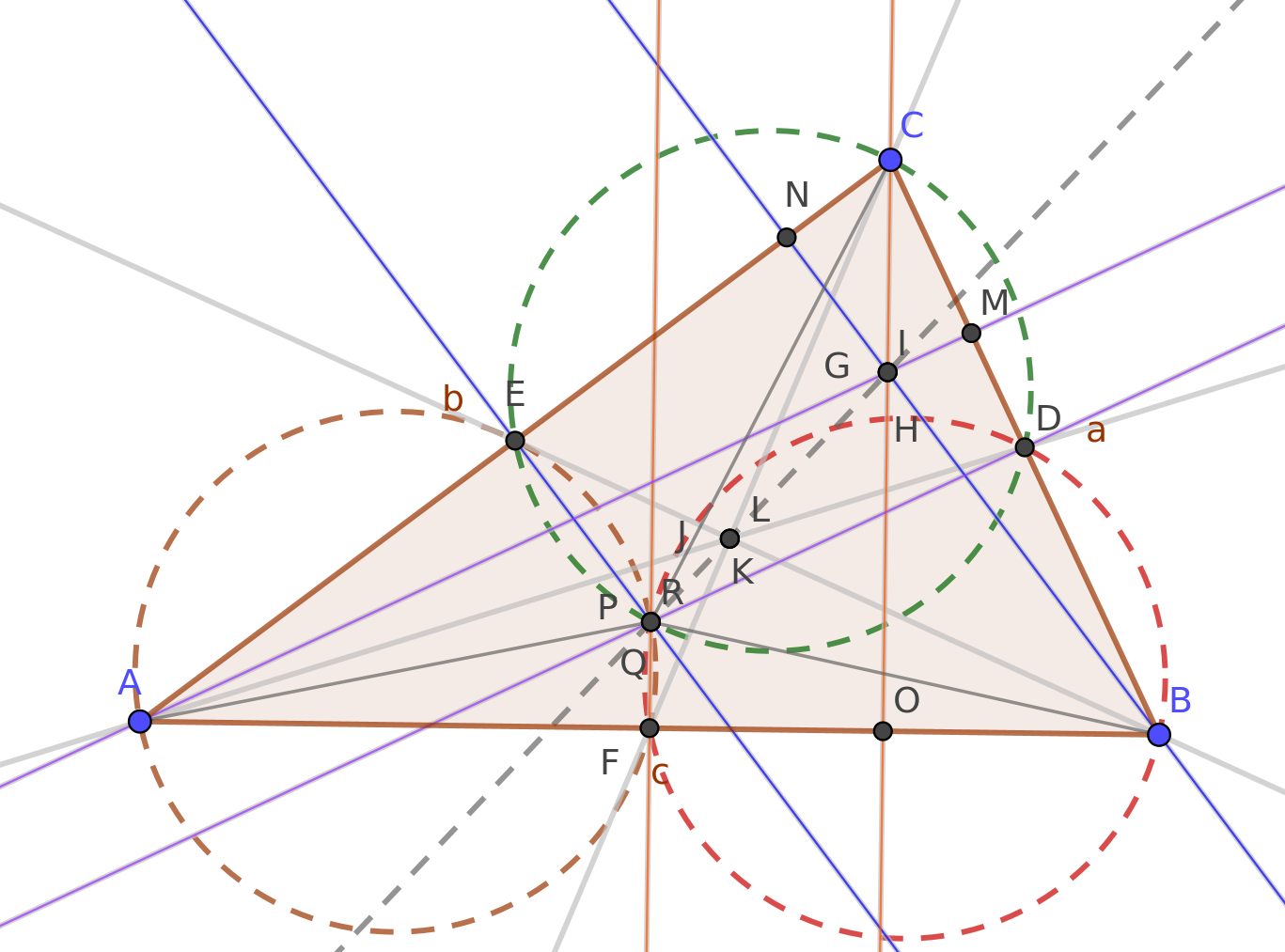}
{\setlength{\fboxsep}{0pt}%
\setlength{\fboxrule}{0.5pt}%
\fbox{\includegraphics[scale=0.4]{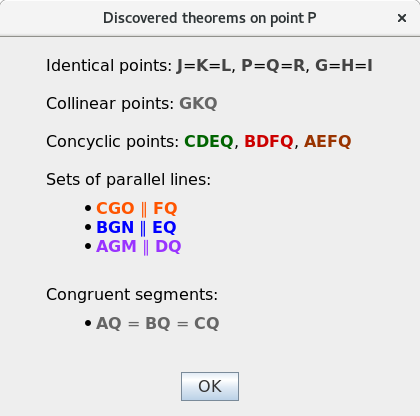}}%
}
\caption{Output of the command \texttt{Discover($P$)}}
\label{EulerLine2}
\end{figure}
The Euler line theorem implicitly includes several simple theorems, including concurrency of the medians
of a triangle ($J=K=L$, the generated points being the pairwise intersections of the medians),
concurrency of the altitudes ($G=H=I$, these points being the pairwise intersections of the altitudes),
and concurrency of the perpendicular bisectors of the altitudes ($P=Q=R$, pairwise intersections as above).

\subsection{Nine-point circle}

The nine-point circle passes through nine significant
points of an arbitrary triangle, namely:
\begin{itemize}
\item the midpoint of each side of the triangle,
\item the foot point of each altitude,
\item the midpoint of the line segment from each vertex of the triangle to the orthocenter.
\end{itemize}
The problem setting is shown in Fig.~\ref{9pc1}.
\begin{figure}\centering
\includegraphics[scale=0.75]{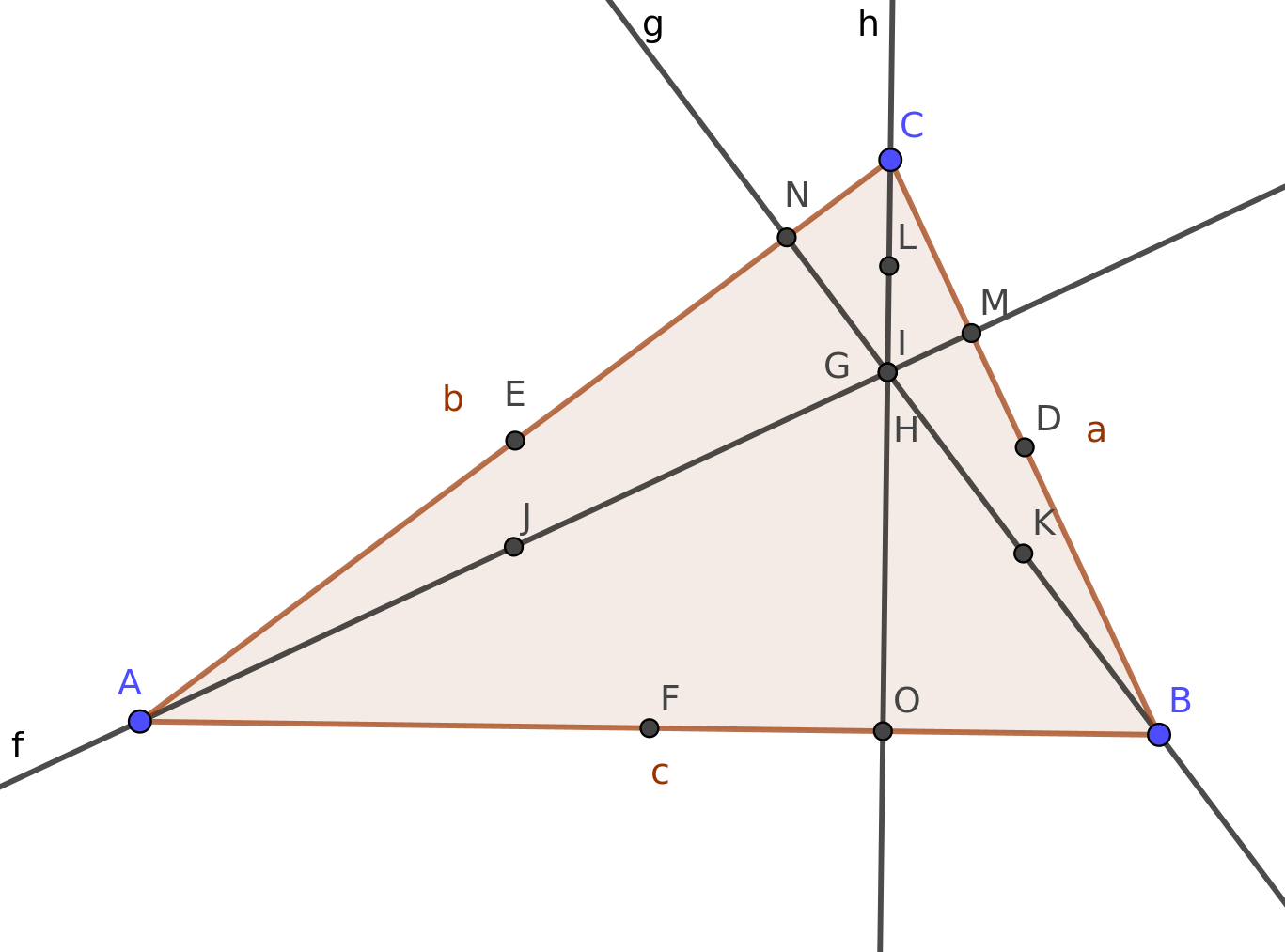}
\caption{Problem setting for the nine-point circle}
\label{9pc1}
\end{figure}
With discovery on point $D$, the appropriate theorems are reported in Fig.~\ref{9pc2}.
\begin{figure}\centering
\includegraphics[scale=0.75]{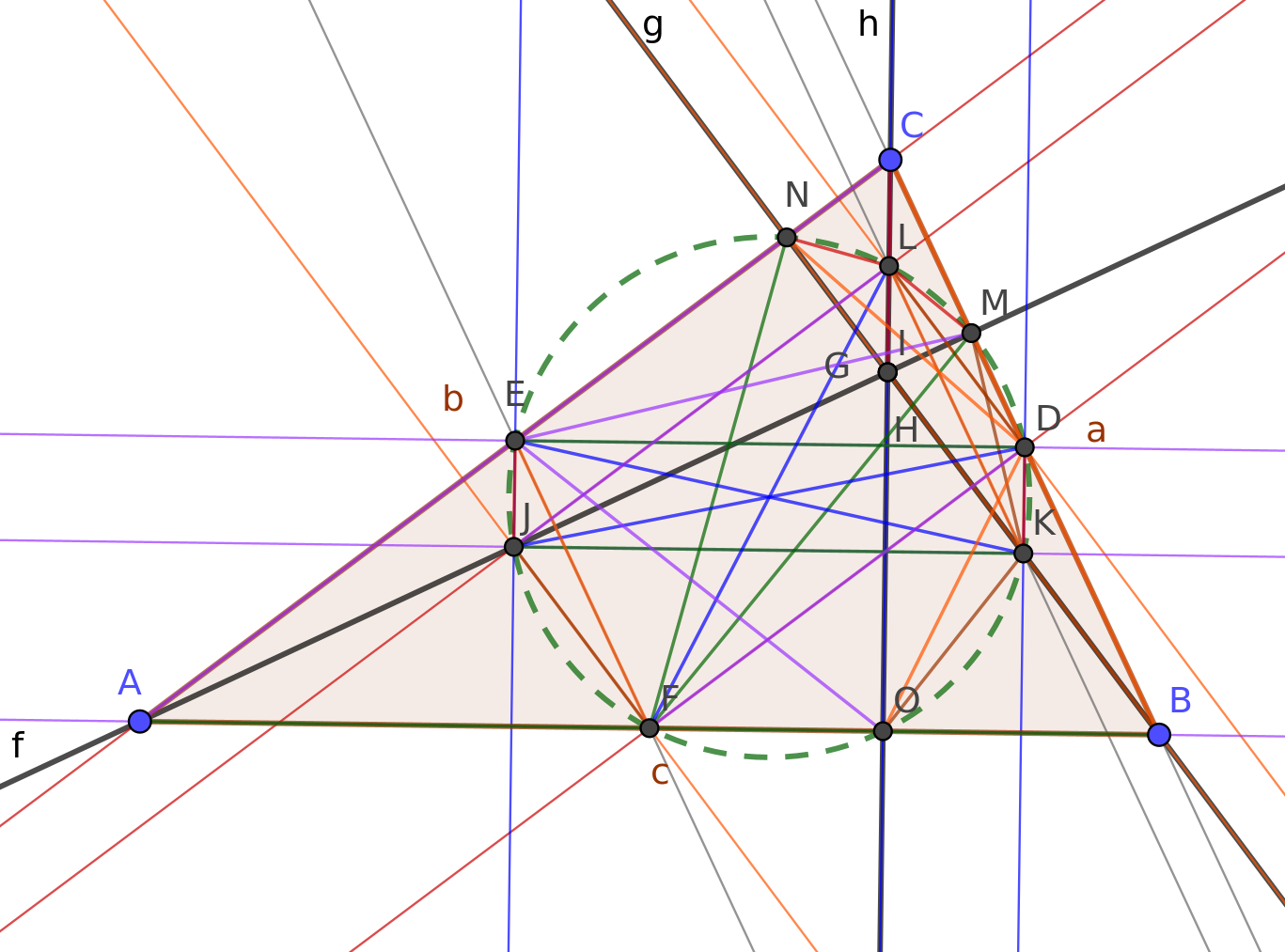}
{\setlength{\fboxsep}{0pt}%
\setlength{\fboxrule}{0.5pt}%
\fbox{\includegraphics[scale=0.4]{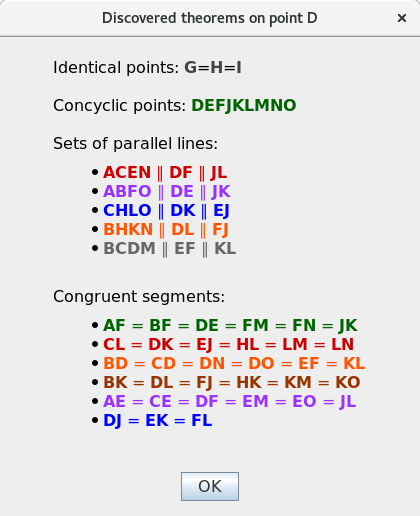}}%
}
\caption{Output of the command \texttt{Discover($D$)}}
\label{9pc2}
\end{figure}
The nine-point circle theorem implicitly includes several other simple theorems. In addition, the graphical result
suggests further theorems: segments $DJ$, $EK$, and $FL$ are congruent and concurrent;
these three segments are also diameters of the nine-point circle; and
their intersection designates the center of the nine-point circle. 
By using another discovery this can be confirmed.

\subsection{A contest problem}

In 2010, at the 51st International Mathematics Olympiad in Astana, Kazakhstan, the following shortlisted
problem was proposed by United Kingdom:
\begin{quotation}
Let $ABC$ be an acute triangle with $D$, $E$, $F$ the feet of the altitudes lying on $BC$, $CA$, $AB$
respectively. One of the intersection points of the line $EF$ and the circumcircle is $P$. The lines
$BP$ and $DF$ meet at point $Q$. Prove that $AP=AQ$.
\end{quotation}
After constructing the according figure with GeoGebra Discovery (Fig.~\ref{contest1}),
\begin{figure}\centering
\includegraphics[scale=0.66]{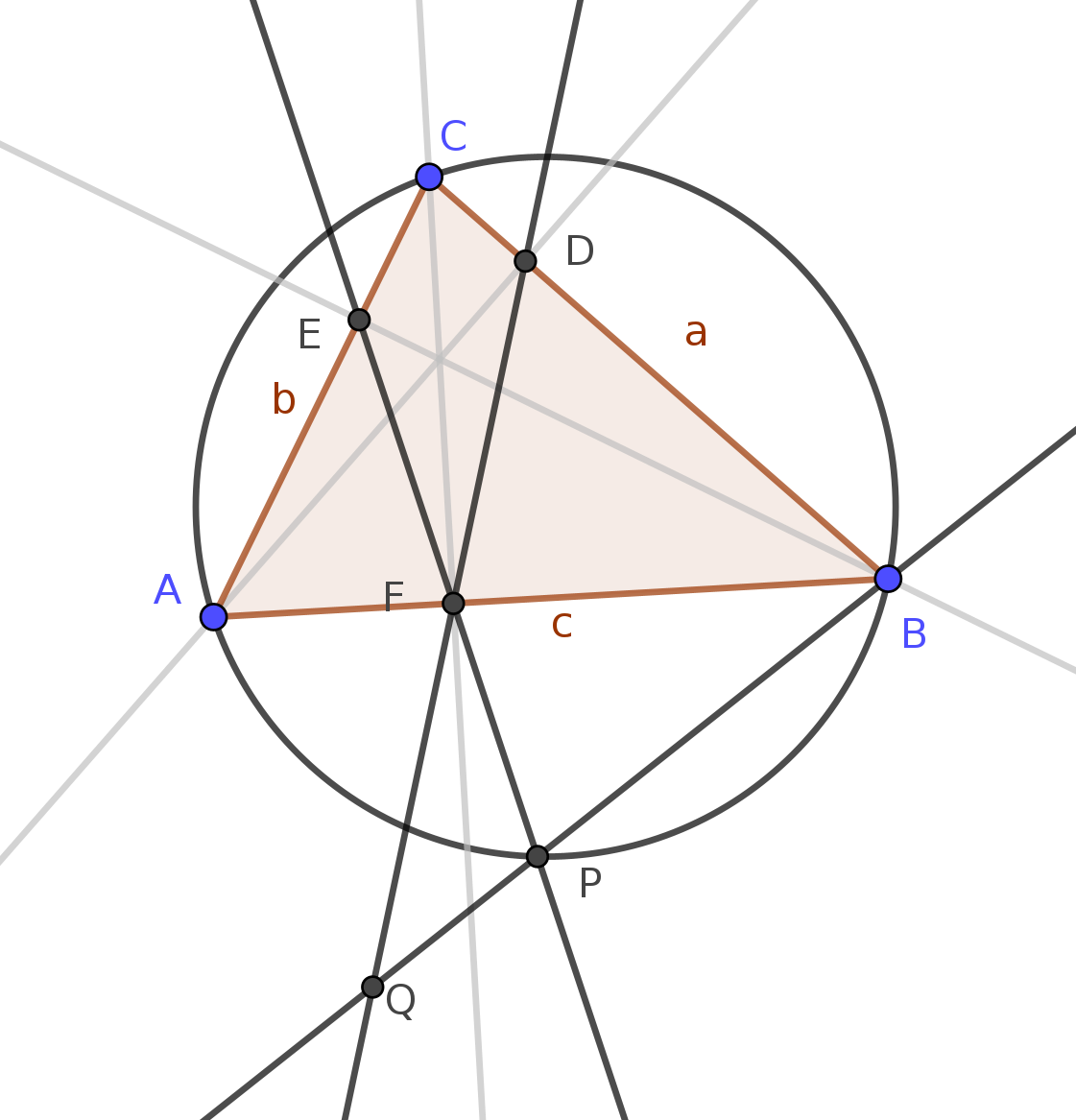}
\caption{Problem setting for a shortlisted contest problem at IMO 2010}
\label{contest1}
\end{figure}
we start discovery on point $Q$. The discovered theorems appear in Fig.~\ref{contest2}.
We learn a few unexpected properties: $DP\parallel EQ$, and the concyclicity of points $C$, $D$, $P$, $Q$,
and $A$, $F$, $P$, $Q$.
\begin{figure}\centering
\includegraphics[scale=0.66]{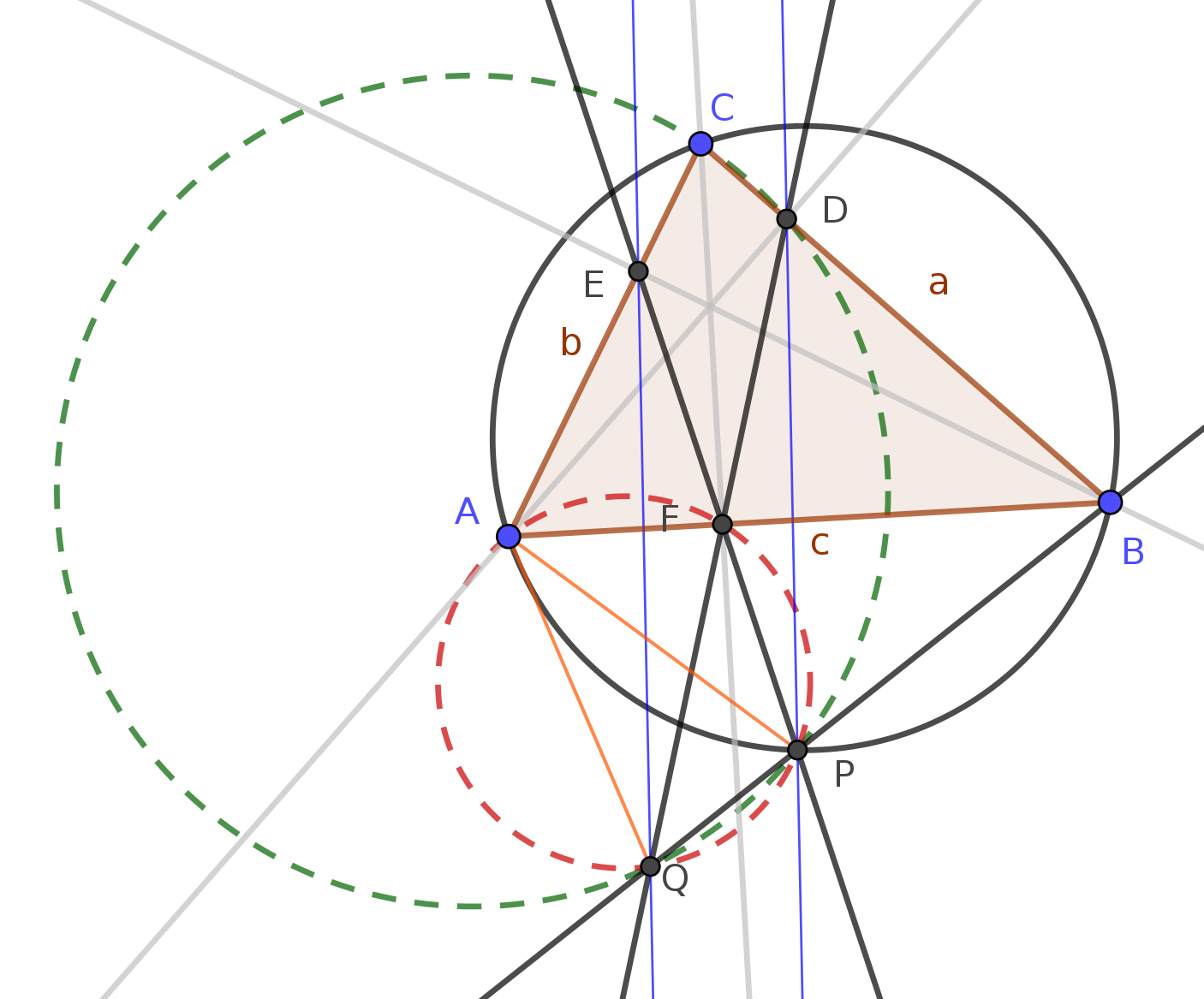}
{\setlength{\fboxsep}{0pt}%
\setlength{\fboxrule}{0.5pt}%
\fbox{\includegraphics[scale=0.4]{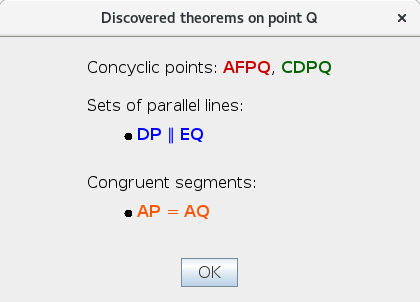}}%
}
\caption{Output of the command \texttt{Discover($Q$)}}
\label{contest2}
\end{figure}

\subsection{Pappus's hexagon theorem}

Consider two sets of collinear triplets $A$, $B$, $E$;
and $C$, $D$, $F$. The intersection points $G=AD\cap BC$, $H=AF\cap CE$, $I=BF\cap DE$
are created. Pappus's hexagon theorem (Fig.~\ref{pappus1}) claims that
the points $G$, $H$ and $I$ are collinear (in general, after assuming certain non-degeneracy conditions).
\begin{figure}\centering
\includegraphics[scale=0.75]{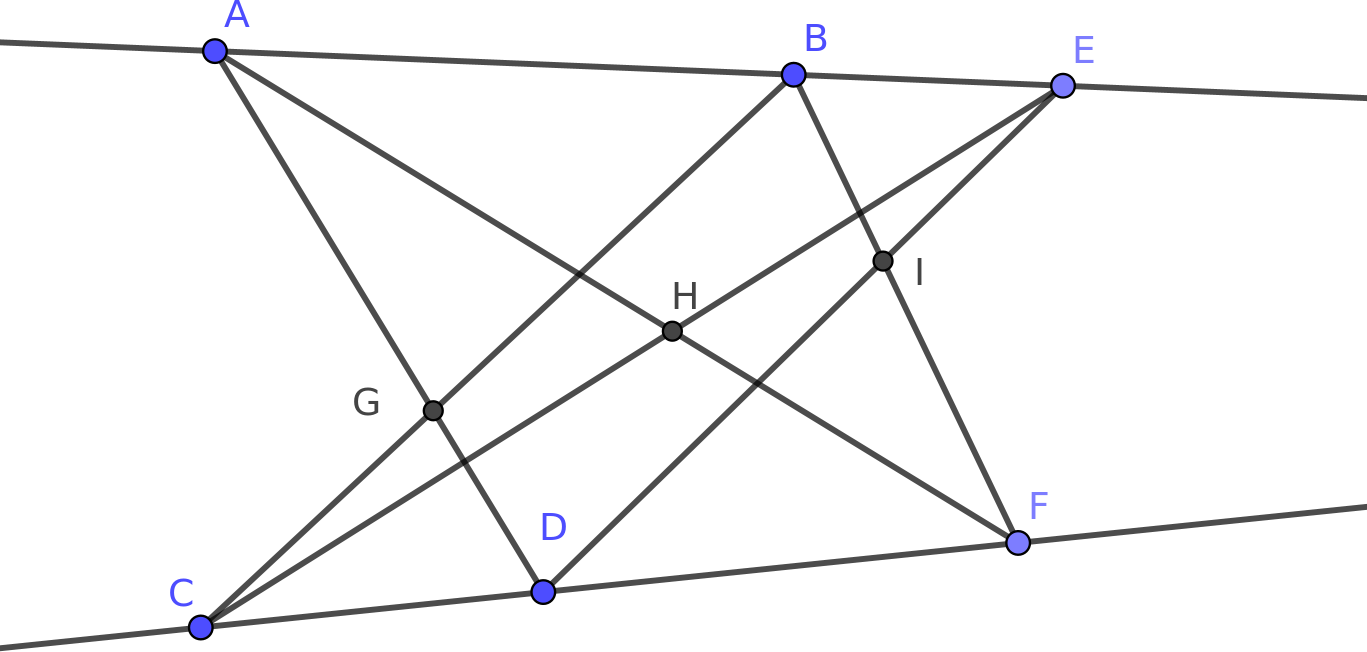}
\caption{Problem setting for Pappus's hexagon theoreom}
\label{pappus1}
\end{figure}
With discovery on point $G$, the theorem is reported in Fig.~\ref{pappus2}.
This final example is more commonly discussed at the university level, rather than in secondary school.
\begin{figure}\centering
\includegraphics[scale=0.75]{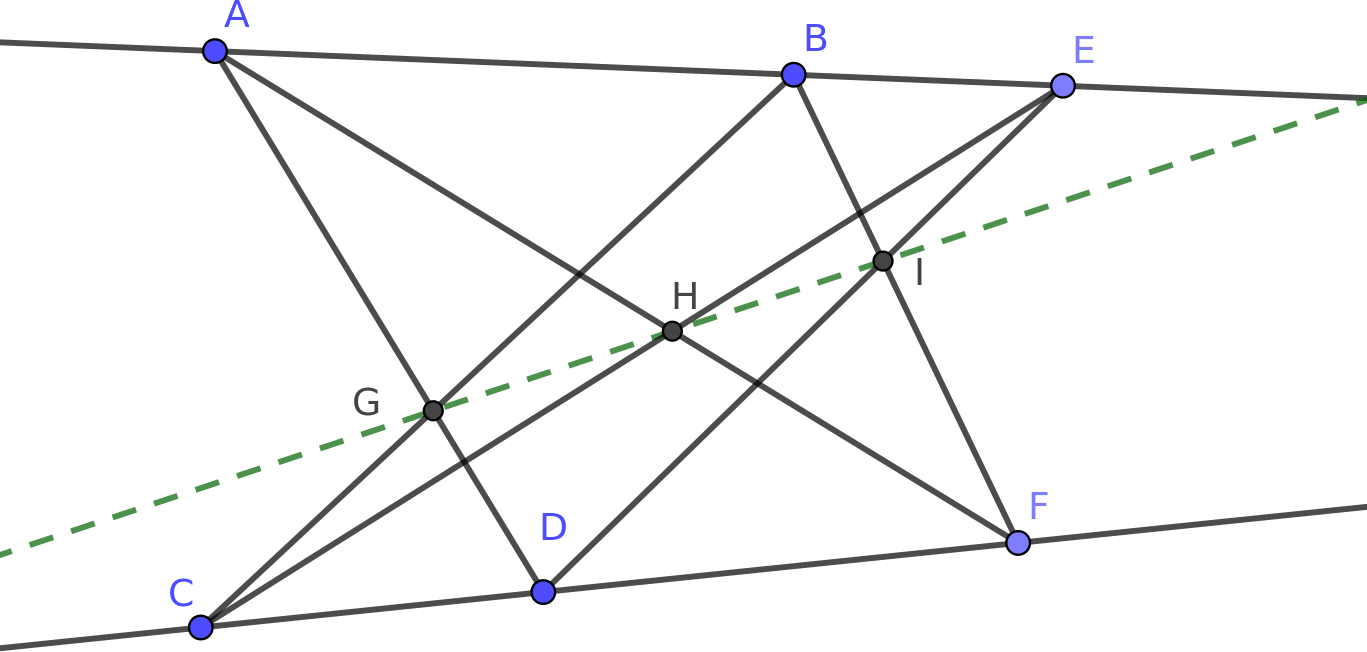}
{\setlength{\fboxsep}{0pt}%
\setlength{\fboxrule}{0.5pt}%
\fbox{\includegraphics[scale=0.4]{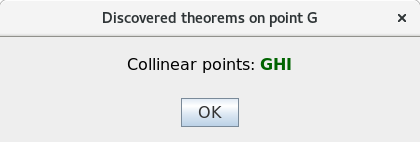}}%
}
\caption{Output of the command \texttt{Discover($G$)}}
\label{pappus2}
\end{figure}

\begin{center}
*
\end{center}
As a final note we highlight that the user interface for the geometric discovery
is designed to be easy for non-experts as well. One does not need to use anything
else but the mouse pointer to obtain all the information.

\section{Discussion}\label{sec4}

\subsection{Trivial statements and theorems}
In Fig.~\ref{rel-midline} the collinearity of points $B$, $C$ and $D$ and of points $A$, $C$ and $E$
were not reported. This is intentional:
by defining $D$ as the midpoint of $BC$ we implicitly assumed this collinearity,
so it does not make any sense to reiterate this. Therefore, it seems
useful to make a distinction between \textit{trivial statements} and \textit{theorems}.

The question of which properties are considered trivial or not is at some level a judgment call.
In Fig.~\ref{rel-midline} most users may regard the information
$BD=CD$ as trivial, with $D$ being the midpoint of $BC$. On the other hand,
for beginners this information may still be useful.

At the moment GeoGebra Discovery maintains some background information if the obtained
theorem is to be displayed or not. For example, in Fig.~\ref{parallelogram2}, the collinearity of
$P_2$, $P_4$ and $P_6$
is considered trivial and not displayed, but the fact that $P_5=P_6$ is presented as non-trivial.
By considering both of these ideas, the collinearity of $P_2$, $P_4$ and $P_5$
could be considered either trivial or non-trivial---currently it is considered as trivial and not shown.
The decision process for such questions should be clarified in the future.

%

\subsection{Combinatorial explosion and computational complexity}
Despite the large number of possible statements, the combinatorial complexity is still polynomial,
because from a given set of input objects $P_1,P_2,\ldots,P_n$ we need to select just at most four objects (four objects are required to confirm concyclicity.) On the other hand,
by using the classes of the equivalence relations, the number of statements to be checked can be decreased significantly.

For each possible statement, a numerical check is first performed. We assume that this is always
successful when a generally true statement is about to check. Unfortunately, in reality this is not always the case,
because for some exotic coordinates, the numerical check can be completely misleading.
For example, some very large numbers can result in numerically unstable computations.
Regardless, if a numerical check is positive, then the statement is added
to the list of conjectures, but if it is negative, no conjecture is registered.
As a consequence, while our implementation
may miss some true statements (due to numerical errors), it will not output false statements.

For each conjecture, a symbolic check will be performed. If the symbolic check is positive,
then the statement will be saved as a theorem. If the symbolic check is negative, then the statement will be
removed from the list of conjectures. If the symbolic check cannot decide if a conjecture is
true or false, the conjecture is removed from the list.

A special case of a conjecture is $P_1=P_2$ for each two geometric points. If this conjecture cannot be
proven or disproven symbolically, then the discovery process will be completely stopped and the user will be notified
that the construction must be redrawn in a different way---otherwise no output can be produced. This
exception is required to keep the internal data consistent.

Symbolic checks usually require more time than numerical verifications. The underlying computation
uses Gr\"obner bases that require at most double exponential time of the number of variables \cite{mayrmeyer82} according
to the given figure. Usually, the number of variables are double the number of
geometric points in the figure (since there are two coordinates for each).

GeoGebra internally sets 5 seconds for the maximal execution time of each symbolic check. After timeout
the result of the symbolic check will be undecided.

\subsection{User interface enhancements}

GeoGebra is designed with a straightforward user interface that asks the users
no questions if possible. However, its usability could be improved for situation when
the user wishes to limit the output by filtering or excluding certain relationships.

Currently only points can be investigated. In a future version a set of points, segments,
lines, circles or a set of these should be permitted as input.

Currently the computation process cannot be interrupted by the user. Given a large number of points
in the figure, the calculation can be time consuming. For example, investigating the relationships
of a regular 20-gon may require about 4 minutes on a modern personal computer (in our test
a Lenovo ThinkPad E480 with 8$\times$i7, 16 GB RAM, Ubuntu Linux 18.04, was used). See Fig.~\ref{20gon} for the output.

\begin{figure}\centering
\includegraphics[scale=0.4]{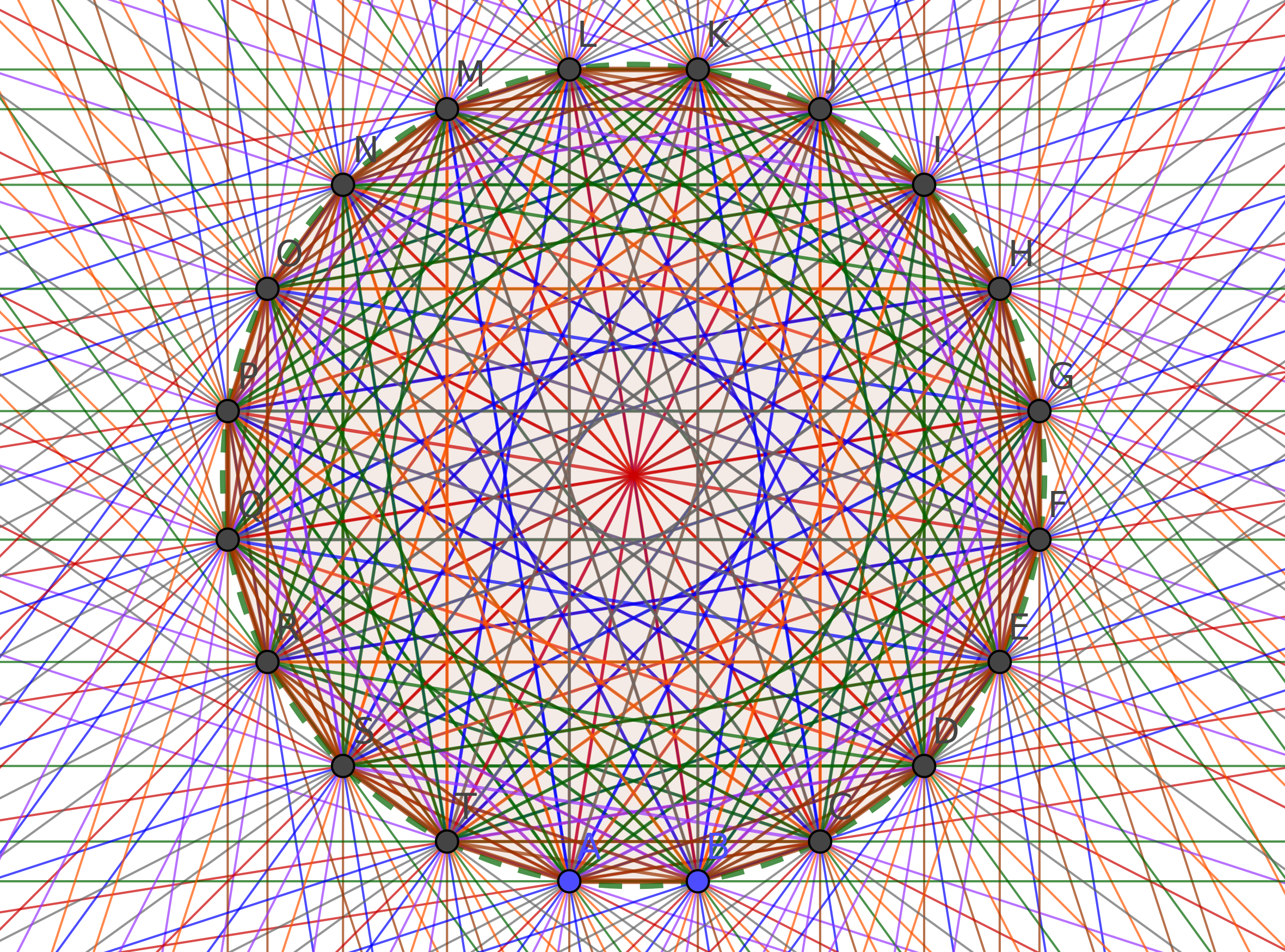}\\
{\setlength{\fboxsep}{0pt}%
\setlength{\fboxrule}{0.5pt}%
\fbox{\includegraphics[width=\textwidth]{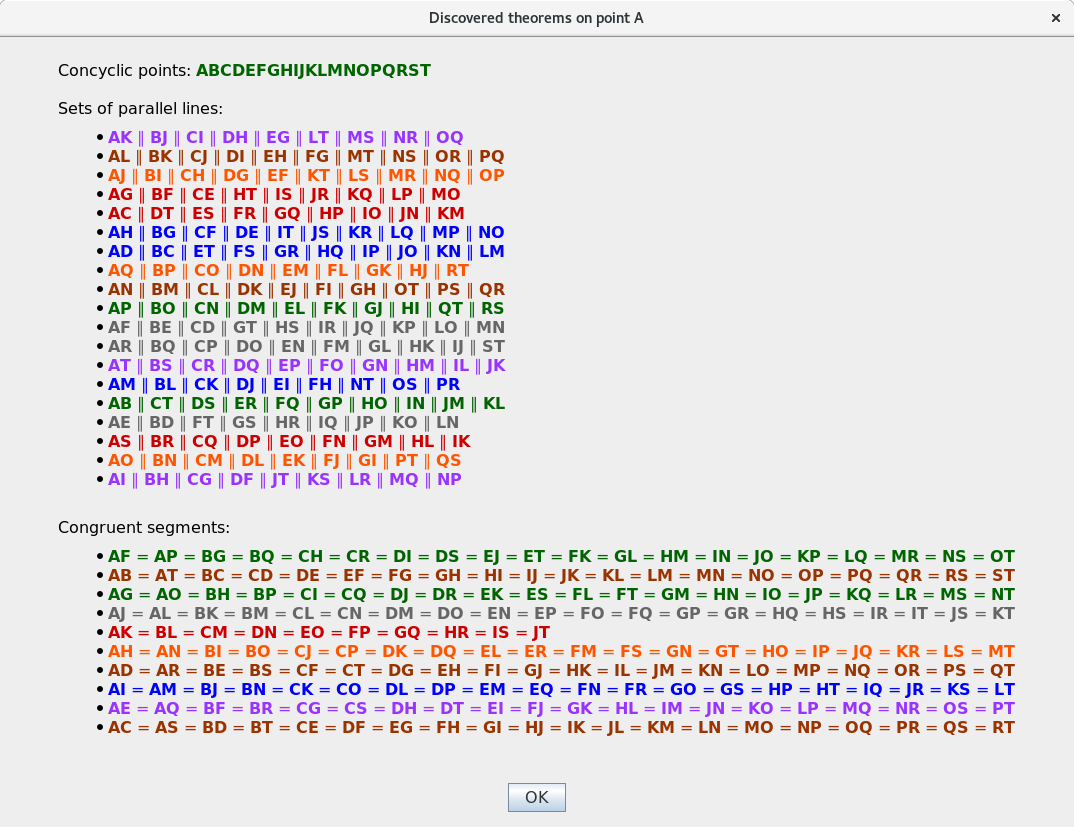}}%
}
\caption{Output for \texttt{Discover($A$)} in a regular 20-gon}
\label{20gon}
\end{figure}

The version that is based on GeoGebra Classic 5 performs better than the one on Classic 6---the
latter is a web implementation of the GeoGebra application and uses a WebAssembly compilation
of the computer algebra system Giac. Even if the code is reasonably fast as embedded code
in a web page, this latter version underperforms the native technology: the same hardware
is unable to handle the input of the regular 20-gon, and the browser tab crashes
after 12 minutes of computation. (Google Chrome 83 was used for testing.)

\subsection{Colors}
At the moment a limited set of colors is used to highlight parallelism and congruence.
In the future a pre-defined sequence of distinguishable colors should be added
to GeoGebra Discovery---for example, at the moment in Fig.~\ref{rel-hexagon} the same black color is used to 
highlight different sets of parallel lines.

\subsection{Perpendicular lines}
Perpendicular lines play an important role in elementary planar geometry. Their detection and
presentation are not yet implemented in GeoGebra Discovery.
Here we mention that the relationship of perpendicularity
is \textit{not} an equivalence, in contrast to the previous relationships defined in Section \ref{sec2}.
On the other hand,
if $\vec{D}$ and $\vec{E}$ are directions, if $\ell\in\vec{D}$ and $m\in\vec{E}$, the relationship
$\ell\perp m$ implies perpendicularity for all $\ell'\in\vec{D}$ and $m'\in\vec{E}$, that is, $\ell'\perp m'$.

It seems convenient to color perpendicular lines with the same color. So a rectangular grid can be
observed for each pair of directions $\vec{D}$ and $\vec{E}$ whose representative lines are
perpendicular, accordingly. Fig.~\ref{grid} shows an example that includes four rectangular grids
for the parallel diagonals of a regular octagon.

\begin{figure}\centering
\includegraphics[scale=0.75]{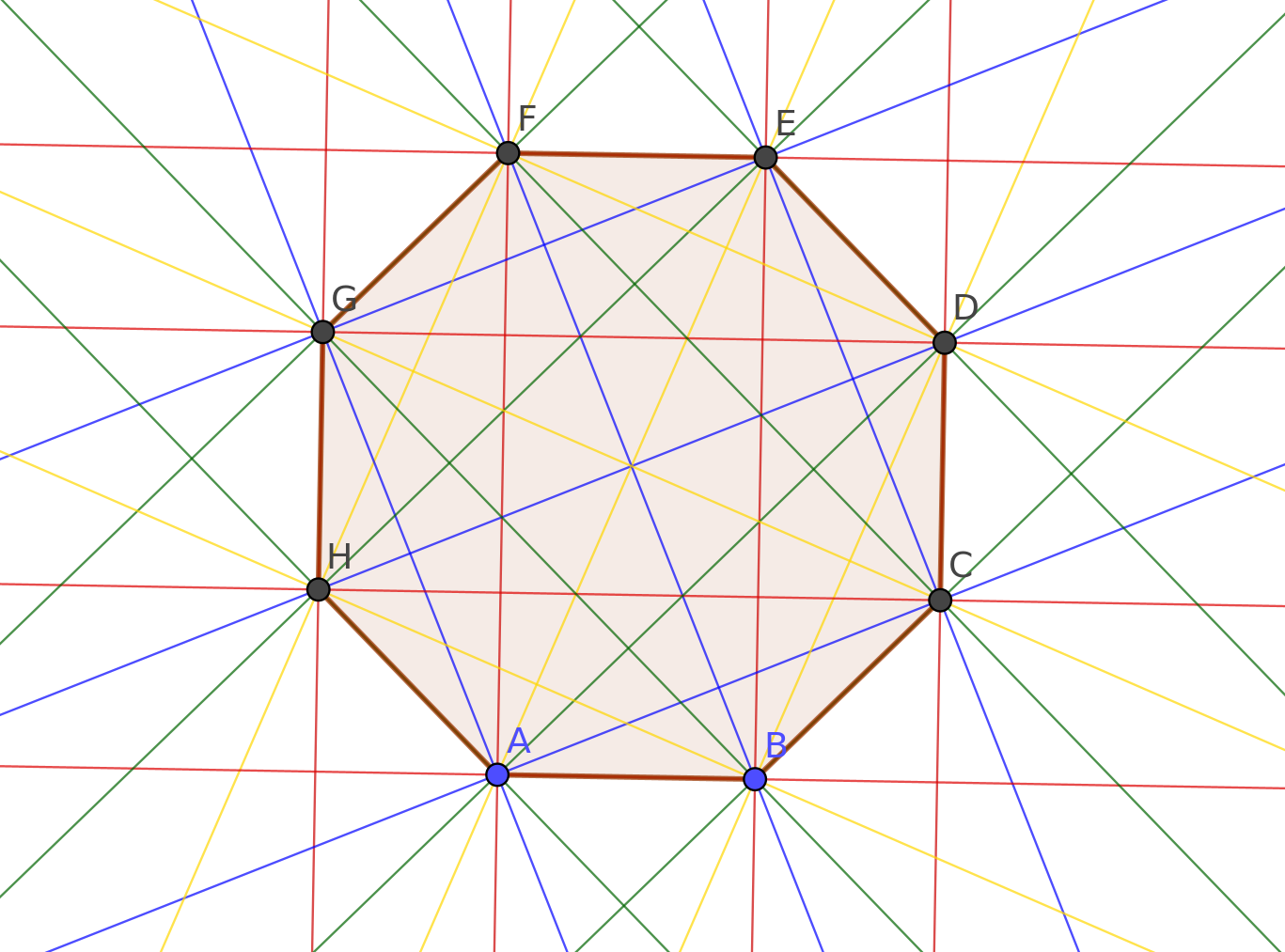}\\
\caption{Four rectangular grids describing the parallel diagonals of a regular octagon}
\label{grid}
\end{figure}

\subsection{Angles}
In a complex algebraic geometry setting, the study of angles is not as straightforward as investigating other objects.
For a future version, however, this feature would be an important improvement.

By combining algebraic and pure geometric observations, however, simple theorems on angle equality could
be easily detected. For example, Fig.~\ref{contest2} states that points $A$, $F$, $P$, $Q$ are concyclic.
The inscribed angle theorem automatically implies $\angle QAP=\angle QFP$, among others.

\subsection{Stepwise suggestions}

Prior research (see \cite[p.~46]{matematech}) proposed that collecting the interesting new objects
in a figure could be done stepwise, similarly to GeoGebra's former feature ``special objects.''
For our midline theorem example (Fig.~\ref{midline1}), this meant that after constructing the triangle $ABC$, and then
midpoint $D$, the segments $BD$ and $CD$ were automatically shown by the system. The user could then accept these newly generated segments or remove them from the system.
Then, by creating midpoint $E$, the system could show lines $AB$ and $DE$ to visualize parallelism.

Actually, the ``special objects'' feature was recently removed from GeoGebra after some negative feedback from
the community---many users found this feature confusing.
As a consequence, adding stepwise suggestions in GeoGebra Discovery remains a question for future research.

\subsection{Benchmarks}

There is no benchmarking suite for the \texttt{Discover} command yet. This should be addressed
in the next phase of the development.

\section{Related work}

We now discuss several projects that share some similarity to GeoGebra Discover but differ significantly
in meaningful ways.

First of all, GeoGebra Discovery is not the first tool that systematically displays confirmed theorems in a geometric
figure. We refer the reader to 
\begin{itemize}
\item Zlatan Magajna's \textit{OK Geometry} \cite{Magajna2011}
(available at \url{www.ok-geometry.com})
and
\item Jacques Gressier's \textit{G\'eom\'etrix}
(available at \url{geometrix.free.fr}).
\end{itemize}
These systems are available free of charge, but without the source code.
On the other hand, GeoGebra Discovery focuses on an intuitive user interface
and proofs in the most mathematical sense.

Second, we highlight that there is a growing interest in creating algorithms related to success completion of 
secondary school or undergraduate mathematics entrance exams. (See \cite{robot-china}, \cite{text-diagram}, \cite{japan-todai}, among others.)
Sometimes these projects rely significantly on techniques used in the underlying computational methods.
Also, these projects are often related to artificial intelligence and Big Data rather than to
computational mathematics.

Third, we mention a theoretical issue. The idea to store a geometric point only once if it is identical to another one was previously described in Kortenkamp's work \cite[9.3.1]{ulli99}. This 
concept is a main design element in the dynamic geometry software Cinderella, which
never stores a geometric point twice if the two variants are identical in general.

GeoGebra has a different design concept by allowing the user an arbitrary number of identical points to
be defined. From the theorem prover's point of view, GeoGebra's concept is more difficult to handle,
and a kind of translation is required to have a different data structure by using the
concepts from Section \ref{sec2}.

Also, we note that GeoGebra Discovery proves the truth in a different manner from Cinderella,
with Cinderella using a probabilistic method,
and GeoGebra Discovery literally \textit{proving} all the deduced facts.

\section{Conclusion}

We described a prototype of the \texttt{Discover} command that is available in an experimental
version of GeoGebra, called GeoGebra Discovery. Our current implementation can be directly
downloaded from \url{https://github.com/kovzol/geogebra/releases/tag/v5.0.591.0-2020Jul16}.

Our work is still in progress, as noted with the issues listed in Section \ref{sec4}.

\section{Acknowledgments}

The \texttt{Discover} command is a result of a long collaboration of several researchers.
The project was initiated by Tom\'as Recio in 2010, and several other researchers joined,
including Francisco Botana and M.~Pilar V\'elez, to name just the most prominent collaborators.
The development and research work was continuously monitored and supported by the GeoGebra Team.
Special thanks to Markus Hohenwarter, project director of GeoGebra.

The work was partially supported by a grant MTM2017-88796-P from the
Spanish MINECO (Ministerio de Economia y Competitividad) and the ERDF
(European Regional Development Fund).

\bibliography{kovzol,external}

\begin{thebibliography}{10}

\bibitem{geogebra-discovery}
Kov\'acs, Z.:
\newblock {GeoGebra} {Discovery}.
\newblock A GitHub project (2020)
  \url{https://github.com/kovzol/geogebra-discovery}.

\bibitem{ag}
Botana, F., Kov\'acs, Z., Recio, T.:
\newblock Automated {G}eometer.
\newblock A GitHub project (2018) \url{https://github.com/kovzol/ag}.

\bibitem{adg-ag}
Botana, F., Kov\'acs, Z., Recio, T.:
\newblock Automated {G}eometer, a web-based discovery tool.
\newblock Presentation at ADG-12, Nanning, China (2018)

\bibitem{aisc-ag}
Botana, F., Kov\'acs, Z., Recio, T.:
\newblock Towards an automated geometer.
\newblock Presentation at AISC-13, Suzhou, China (2018)

\bibitem{LNAI11110-ag}
Botana, F., Kov\'acs, Z., Recio, T.:
\newblock Towards an automated geometer.
\newblock In Fleuriot, J., Wang, D., Calmet, J., eds.: Artificial Intelligence
  and Symbolic Computation. Volume 11110 of Lecture Notes in Artificial
  Intelligence., Springer International Publishing (2018)  215--220

\bibitem{song}
Chen, X., Song, D., Wang, D.:
\newblock Automated generation of geometric theorems from images of diagrams.
\newblock Annals of Mathematics and Artificial Intelligence \textbf{74} (2015)
  333--358

\bibitem{Chou_1987}
Chou, S.C.:
\newblock {Mechanical Geometry Theorem Proving}.
\newblock Springer Science $+$ Business Media (1987)

\bibitem{rmc-top}
Kov\'acs, Z., Recio, T., V\'elez, M.P.:
\newblock Detecting truth, just on parts.
\newblock Revista Matem\'atica Complutense \textbf{32} (2019)  451--474

\bibitem{mayrmeyer82}
Mayr, E., Meyer, A.:
\newblock The complexity of the word problem for commutative semigroups and
  polynomial ideals.
\newblock Advances in Mathematics \textbf{46} (1982)  305--329

\bibitem{matematech}
Kov{\'a}cs, Z.:
\newblock Towards a new {GeoGebra} {Geometry} {App}.
\newblock Presentation at MatemaTech Seminar for teachers, \v{C}esk\'e
  Bud\v{e}jovice, Czechia (2019)

\bibitem{Magajna2011}
Magajna, Z.:
\newblock An observation tool as an aid for building proofs.
\newblock The Electronic Journal of Mathematics and Technology \textbf{5}
  (2011)  251--260

\bibitem{robot-china}
Fu, H., Zhang, J., Zhong, X., Zha, M., Liu, L.:
\newblock Robot for mathematics college entrance examination.
\newblock In: Electronic Proceedings of the 24th Asian Technology Conference in
  Mathematics, Mathematics and Technology, LLC (2019)

\bibitem{text-diagram}
Seo, M., Hajishirzi, H., Farhadi, A., Etzioni, O., Malcolm, C.:
\newblock Solving geometry problems: Combining text and diagram interpretation.
\newblock In: Proceedings of the 2015 Conference on Empirical Methods in
  Natural Language Processing. (2015)  1466--1476

\bibitem{japan-todai}
Fujita, A., Kameda, A., Kawazoe, A., Miyao, Y.:
\newblock Overview of {T}odai robot project and evaluation framework of its
  {NLP}-based problem solving.
\newblock In: Proceedings of the Ninth International Conference on Language
  Resources and Evaluation (LREC'14). (2014)  2590--2597

\bibitem{ulli99}
Kortenkamp, U.:
\newblock {Foundations of Dynamic Geometry}.
\newblock PhD thesis, ETH Z\"urich (1999)

\end{thebibliography}

\end{document}